%% file: root.tex
\documentclass[letterpaper, 10 pt, conference]{ieeeconf}  % Comment this line out if you need a4paper

\IEEEoverridecommandlockouts                              % This command is only needed if 
\usepackage{graphicx,algorithm,algpseudocode,amsthm,amsmath,amsfonts,verbatim,mathtools,enumerate,bm,dsfont,xcolor}
\usepackage{subcaption}
\def\p(#1|#2){p(#1\,|\,#2)}

\usepackage{multicol}
\definecolor{mycitecolor}{RGB}{71, 191, 38}
\definecolor{mylinkcolor}{RGB}{40, 115, 201}

\usepackage[bookmarks=true, colorlinks, citecolor=mycitecolor,linkcolor=mylinkcolor,urlcolor=mycitecolor ]{hyperref}

\usepackage{stfloats}

\usepackage{float}

\input defs.tex

\usepackage{graphicx}

\newcommand{\maulik}[1]{{\textcolor[rgb]{0,0,0}{#1}}}

\title{\LARGE \bf
MultiNash-PF: A Particle Filtering Approach for Computing Multiple Local Generalized Nash Equilibria in Trajectory Games
}

\author{ Maulik Bhatt$^{1}$ \and
Iman Askari$^{2}$ \and
Yue Yu$^{3}$ \and
Ufuk Topcu$^{4}$ \and
Huazhen Fang$^{2}$ \and
Negar Mehr$^{1}$% <-this % stops a space
\thanks{$^{1}$Maulik Bhatt and Negar Mehr are with University of California, Berkeley, CA, 94720, USA {\tt\small\{maulikbhatt,negar\}@berkeley.edu}}%
\thanks{$^{2}$Iman Askari and Huazhen Fang are with University of Kansas, Lawrence, KS, 66045, USA
        {\tt\small \{askari,fang\}@ku.edu}}%
\thanks{$^{3}$ Yue Yu is with University of Minnesota, Minneapolis, 55455, USA
        {\tt\small yuey@umn.edu}}        
\thanks{$^{3}$ Ufuk Topcu is with University of Texas at Austin, TX, 78712, USA
        {\tt\small utopcu@utexas.edu}}
\thanks{\maulik{This research involving human subjects was approved by the Committee for Protection of Human Subjects under Protocol No. 2024-11-18012.}}
\thanks{\maulik{This work was supported by the National Science Foundation, under grants ECCS-2438314 CAREER Award, CNS-2423130, and CCF-2423131.}}
}

\begin{document}

\maketitle
\thispagestyle{empty}
\pagestyle{empty}

%%%%%%%%%%%%%%%%%%%%%%%%%%%%%%%%%%%%%%%%%%%%%%%%%%%%%%%%%%%%%%%%%%%%%%%%%%%%%%%%
\begin{abstract}

Modern robotic systems frequently engage in complex multi-agent interactions, many of which are inherently multi-modal, i.e., they can lead to multiple distinct outcomes. To interact effectively, robots must recognize the possible interaction modes and adapt to the one preferred by other agents. In this work, we propose MultiNash-PF, an efficient algorithm for capturing the multimodality in multi-agent interactions. We model interaction outcomes as equilibria of a game-theoretic planner, where each equilibrium corresponds to a distinct interaction mode.
Our framework formulates interactive planning as Constrained Potential Trajectory Games (CPTGs), in which local Generalized Nash Equilibria (GNEs) represent plausible interaction outcomes. We propose to integrate the potential game approach with implicit particle filtering, a sample-efficient method for non-convex trajectory optimization. We utilize implicit particle filtering to identify the coarse estimates of multiple local minimizers of the game's potential function. MultiNash-PF then refines these estimates with optimization solvers, obtaining different local GNEs. We show through numerical simulations that MultiNash-PF reduces computation time by up to 50\% compared to a baseline. We further demonstrate the effectiveness of our algorithm in real-world human-robot interaction scenarios, where it successfully accounts for the multi-modal nature of interactions and resolves potential conflicts in real-time.

\end{abstract}

%%%%%%%%%%%%%%%%%%%%%%%%%%%%%%%%%%%%%%%%%%%%%%%%%%%%%%%%%%%%%%%%%%%%%%%%%%%%%%%%
\input{intro}
\input{related_work}
\input{trajgames}
\input{filtering}
\input{numerical_simulations}
\input{experiment}
\input{conclusion}

%%%%%%%%%%%%%%%%%%%%%%%%%%%%%%%%%%%%%%%%%%%%%%%%%%%%%%%%%%%%%%%%%%%%%%%%%%%%%%%%

%%%%%%%%%%%%%%%%%%%%%%%%%%%%%%%%%%%%%%%%%%%%%%%%%%%%%%%%%%%%%%%%%%%%%%%%%%%%%%%%



% \section*{ACKNOWLEDGMENT}

% The preferred spelling of the word ÒacknowledgmentÓ in America is without an ÒeÓ after the ÒgÓ. Avoid the stilted expression, ÒOne of us (R. B. G.) thanks . . .Ó  Instead, try ÒR. B. G. thanksÓ. Put sponsor acknowledgments in the unnumbered footnote on the first page.

%%%%%%%%%%%%%%%%%%%%%%%%%%%%%%%%%%%%%%%%%%%%%%%%%%%%%%%%%%%%%%%%%%%%%%%%%%%%%%%%
% \vspace{-0.3cm}
\bibliographystyle{ieeetr}
\bibliography{references}

\end{document}

%% file: defs.tex
\newcommand{\blkdiag}{\mathop{\rm blkdiag}}

\newcommand{\norm}[1]{\left\lVert#1\right\rVert}
\newcommand{\mnorm}[1]{{\left\vert\kern-0.25ex\left\vert\kern-0.25ex\left\vert #1 
    \right\vert\kern-0.25ex\right\vert\kern-0.25ex\right\vert}}

\newtheorem{definition}{Definition} 
\newtheorem{theorem}{Theorem}
\newtheorem{lemma}{Lemma}

%% file: intro.tex
% \vspace{-0.2cm}
\section{Introduction}\label{sec:intro}
% \vspace{-0.1cm}

% Modern-world robotics applications are inherently multi-agent in nature. For instance, autonomous cars on highways need to account for other vehicles while planning their motion. 
Many real-world robotic interactions are inherently multi-modal. For example, in the two-player interaction illustrated in Fig.~\ref{fig:experiment}, the agents need to move towards their goals while avoiding collisions. Here, two outcomes exist: both agents yielding to their right or left. If the two agents act independently, they might choose conflicting actions. For example, one yields to the left and the other to the right, which can potentially result in a collision. What trajectory each agent would prefer to pick depends on the \emph{convention} they might be following. For example, in some cultures, in such settings, the convention is to yield to the right to avoid collisions, while in other cultures, the convention is to yield to the left. Such conventions typically emerge from repeated coordination problems~\cite{alterman2001convention}, and autonomous agents must recognize them and account for the multi-modality of interactions to coordinate with humans effectively. Some recent works have shown that reasoning about the multi-modality of interactions is critical to aligning agents' preferences over different outcomes~\cite{peters2020inference} or helping coordinate agents' actions via recommendation~\cite{im2024coordination}. In this paper, we address this challenge by proposing MultiNash-PF, an algorithm that efficiently identifies multiple likely interaction outcomes in multi-agent domains and enables agents to adapt their plans accordingly.

\begin{figure}
    \centering
    \includegraphics[width=\linewidth]{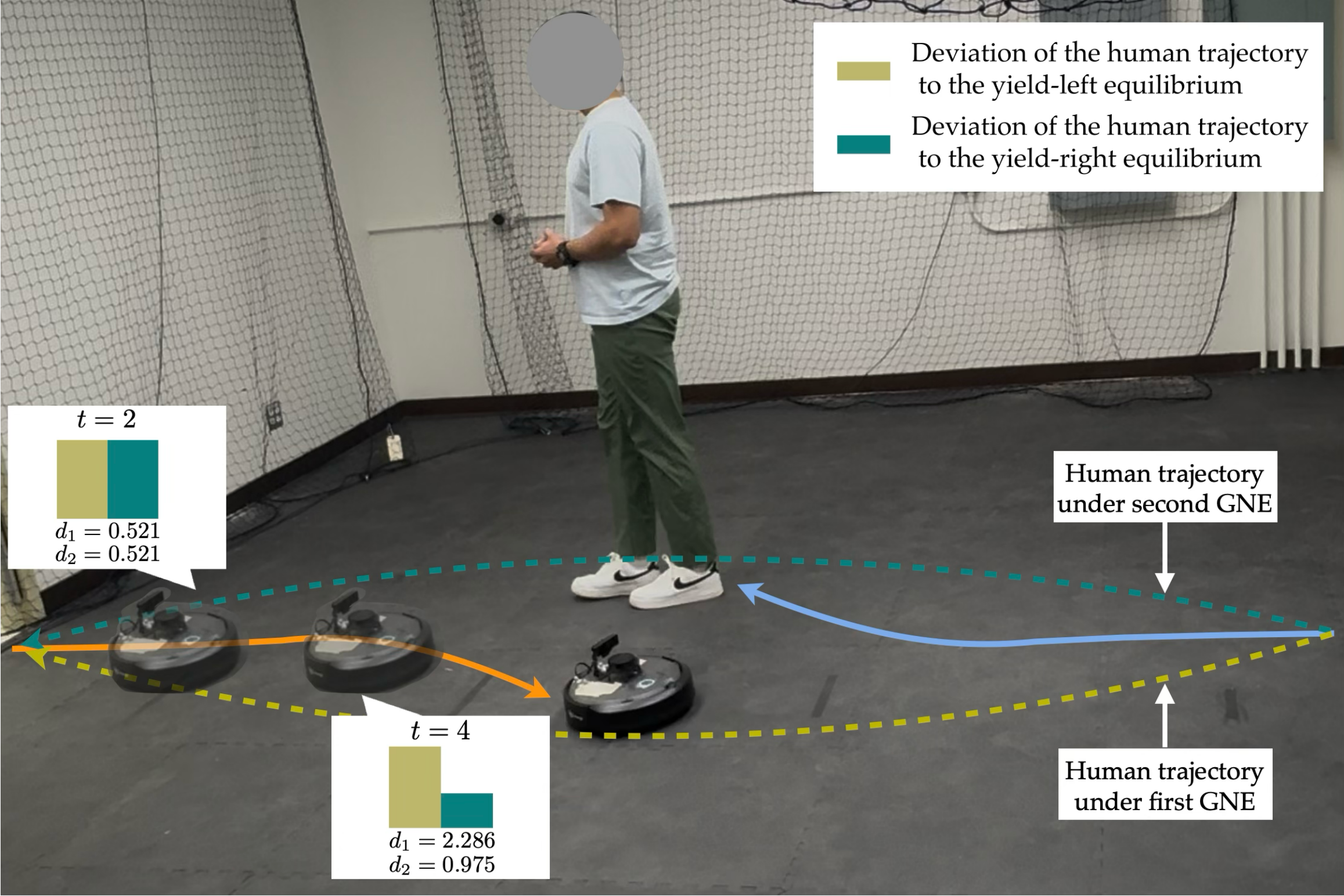}
    \caption{Snapshot of an interaction between a robot and a human when they exchange positions. Using MultiNash-PF, the robot can compute both modes of interaction. At $t = 2$, the robot is uncertain about the preferred mode of the human since the human trajectory's distance from both equilibria is almost the same. At $t = 4$, the human trajectory is closer to the yield-right GNE; therefore, the robot chooses the corresponding GNE and avoids collision with the human.}
    \label{fig:experiment}
    % \vspace{-0.7cm}
\end{figure}

We formulate the problem of computing multi-agent interaction outcomes as a multi-agent trajectory planning problem. To formalize interactions, we use dynamic game theory~\cite{bacsar1998dynamic}, which has been shown to be powerful in developing interactive and socially compliant robots \cite{galati2022game, fridovich2020efficient, sadigh2016planning, bhatt2023strategic} where it has been shown that the interaction outcomes are captured through \emph{generalized Nash equilibrium} (GNE). When multiple agents interact, a GNE corresponds to a joint action such that no agent has incentives to deviate from its equilibrium actions unilaterally. Several recent works have shown great success in utilizing GNE for modeling interaction outcomes in interactive domains~\cite{dreves2018generalized,wang2020game, bhatt2023efficient, jia2023rapid}. However, there typically exist multiple equilibria in interaction games, many of which are equally valid. We argue that each equilibrium corresponds to a distinct \emph{mode} of interaction. For instance, in Fig.~\ref{fig:experiment}, yielding to the right or left are both acceptable equilibria, representing different Nash equilibria. Consequently, we propose that in order to capture different interaction modes, robots must \emph{reason about multiple interaction equilibria and select the one that aligns with the convention adopted by other agents}. However, this is a non-trivial task as it requires the robot to find all local GNEs. To our best knowledge, there is no principled approach to compute multiple local Nash equilibria in constrained trajectory games other than exhaustive random initializations of the game solvers.

% \begin{figure}[t]
%     \centering
%     \includegraphics[width=\linewidth]{figures/WAFR_1.pdf}
%     \caption{\small The Nash equilibrium trajectories found by \emph{MultiNash-PF} when two unicycle agents change their positions. First, we use the implicit particle filter to discover the coarse estimates of two different equilibria, and then utilize them as a warm starting for IPOPT to obtain the two Nash equilibrium trajectories.}
%     \label{fig: two_unicycle}
%     \vspace{-0.7cm}
% \end{figure}

In this paper, we tackle this challenge by introducing {MultiNash-PF}, a novel framework designed to efficiently compute multiple local GNEs in constrained trajectory games. This, in turn, enables robots to recognize and reason about multiple interaction modes. Our approach integrates 1) constrained potential games~\cite{zazo2016dynamic}, a class of games for which local GNEs can be computed by solving a single constrained optimal control problem, and 2) constraint-aware implicit particle filtering which is a sample-efficient method recently applied to nonconvex trajectory optimization~\cite{askari2023model}, to develop a principled approach for identifying various interaction outcomes. 

Using CPTGs, we can compute a local GNE by minimizing the potential function of the game, which is equivalent to solving a single constrained optimal control problem. We can solve the resulting optimal control problem using off-the-shelf optimization solvers. Then, we use a constraint-aware implicit particle filtering method to identify multiple local minima of the potential minimization problem. This method eliminates the need for random initialization of the optimization solver by efficiently obtaining coarse estimates of different local GNEs. Next, we use the identified minima as initialization for an optimization solver to refine our solution. The solver then computes all the different local GNEs of the original game, which correspond to multiple modes of interaction among the agents.
Leveraging the sample efficiency of the implicit particle filtering method, MultiNash-PF can explore and identify multiple local GNEs faster than random initialization of optimization solvers. We further show the real-time application of MultiNash-PF in a human-robot navigation setting where the robot can reason about the existence of different local GNEs and plan its motion autonomously based on the modality the human prefers.

%% file: related_work.tex
% \begin{figure}[t]
%     \centering
%     % \includegraphics[scale=0.28]{figures/WAFR_2_2.pdf}
%     \includegraphics[width=\linewidth]{figures/WAFR_2_2.pdf}
%     \caption{\small MultiNash-PF identifies all the local GNE trajectories for a challenging two agent trajectory game where two agents swapped positions while avoiding an obstacle. In this scenario, the obstacle radius is greater than the collision avoidance radius of the two agents, leading to the discovery of 6 distinct Nash equilibria. MultiNash-PF effectively obtains coarse estimates of local GNEs through implicit particle filtering, which are then refined using IPOPT to obtain local GNEs.}
%     \label{fig:experiment-2}
%     \vspace{-0.5cm}
% \end{figure}

% \begin{figure}
%     \centering
%     \includegraphics[width=0.8\linewidth,trim={0 0.8cm 0 1cm},clip]{figures/WAFR-2-filtering.pdf}
%     \vspace{-0.7cm}
%     \caption{When using MultiNash-PF on a difficult two agent trajectory game, we first employ an implicit particle filter to obtain coarse estimates of GNEs.}
%     \vspace{-0.7cm}
%     \label{fig:filtering}
% \end{figure}

\vspace{-0.2cm}
\section{Related Works}

\noindent \textbf{Game Theoretic Planning.}
Game-theoretic planning has been widely used for modeling interactive domains. One of the foundational works in game theoretic motion planning was \cite{lavalle2000robot}. A wide range of methods have been proposed to compute different types of equilibria such as Stackelberg equilibria~\cite{liniger2019noncooperative,sadigh2016planning}, Nash equilibria, and generalized Nash equilibria through techniques such as best response dynamics~\cite{spica2020real,wang2019game}, terative linear-quadratic approximations~\cite{fridovich2020efficient,wang2020game,mehr2023maximum}, augmented Lagrangian-based solvers~\cite{le2022algames}, and potential games-based methods~ \cite{zazo2016dynamic,kavuncu2021potential,bhatt2023efficient,jia2023rapid,williams2023distributed,bhatt2025strategic}. However, the shortcomings of all these methods lie in the fact that they find only one Nash equilibrium of the underlying game and ignore the multi-modal nature of the problem that is crucial for coordination among agents. In \cite{peters2020inference}, the authors recognize the existence of multiple equilibria and focus on inferring the equilibrium of other agents in the environment. However, their methodology relies on computing different approximate local Nash equilibria using rollouts of randomly selected initial strategies, which we show in this work performs poorly compared to our method. 

\noindent \textbf{Multiple Local Solutions to a Constrained Optimal Control Problem.}
Optimal control problems that are relevant to robotics are often nonconvex constrained optimization problems, which are solved using gradient-based methods such as interior-point algorithms~\cite{Wachter:2006wt} and sequential quadratic programming~\cite{boggs1995sequential}. These methods, however, are limited to finding a single local solution per initialization.  In the context of identifying all local GNEs, this limitation results in a significant computational burden as numerous initial guesses are required to explore the solution space. To address this, we depart from traditional approaches and adopt a Bayesian inference framework that enables simultaneous recovery of multiple local GNEs. Recent work~\cite{Askari:TRO:2025, askari2021nonlinear, askari2022sampling} has demonstrated that constrained optimization problems can be reformulated as Bayesian inference tasks, enabling estimation algorithms to overcome the limitations of gradient-based solvers. In this setting, particle filtering methods~\cite{doucet2001sequential,sarkka2023bayesian} provide a principled approach to capturing multimodal distributions. We leverage this by employing the implicit particle filter, which improves sample efficiency by focusing particles in high-probability regions of the posterior~\cite{Askari:2022:Auto}.
% This sampling approach relies on the premise that a few particles are required for accurate estimation if the particles are sampled from high-probability regions of the probability distribution function. 
% In addition, the implicit sampling technique can also mitigate the need for a resampling step to avoid filter collapse.

%% file: trajgames.tex
% \vspace{-0.2cm}
\section{Constrained Potential Trajectory Games}
% \vspace{-0.1cm}
\label{sec: traj games}

% We formulate the multi-agent interactive trajectory planning problem as an $N$-agent constrained potential trajectory game (CPTG). 

Let $\mathcal{N}:= \{1,2,\ldots,N\}$ denote the set of agents' indices. 
We assume that the game is played over a finite time horizon of $\tau\in\mathbb{N}$ time-steps, where $\mathbb{N}$ is the set of natural numbers. 
Let \(x^i_t\in\mathbb{R}^{n_i}\) and \(u_t^i\in\mathbb{R}^{m_i}\) denote the state and control input of the \(i\)-th agent at time $t$, respectively. We assume that the state of the \(i\)-th agent evolves as follows:
\vspace{-0.1cm}
\begin{equation}\label{eq:dynamics}
    x_{t+1}^i=f^i(x_t^i, u_t^i),
\end{equation}
\vspace{-0.5cm}

\noindent for all \(t \leq \tau \), where \(f^i:\mathbb{R}^{n_i}\times \mathbb{R}^{m_i}\to\mathbb{R}^{n_i}\) is a continuously differentiable function characterizing the dynamics of the \(i\)-th agent. 
Let $n=\sum_{i=1}^N n_i, m = \sum_{i=1}^N m_i$. We denote the joint state, input, and dynamics of all agents as $x_t\coloneqq \begin{bmatrix}
        (x_t^1)^\top & (x_t^2)^\top &
        \cdots & (x_t^N)^\top 
    \end{bmatrix}^\top$, $u_t\coloneqq \begin{bmatrix}
        (u_t^1)^\top & (u_t^2)^\top &
         \cdots&
        (u_t^N)^\top 
    \end{bmatrix}^\top$, and $f\coloneqq \begin{bmatrix}
        (f^1)^\top & (f^2)^\top &
         \cdots&
        (f^N)^\top 
    \end{bmatrix}^\top$ respectively.
Note that each $f^i$ is a vector-valued dynamics function for each agent. Therefore, $f:\mathbb{R}^n\times\mathbb{R}^m\to \mathbb{R}^n$ will be a vector-valued dynamics function for the joint system.

We assume that each agent has to satisfy some constraints that couple different agents' states and control inputs. Let $g:\mathbb{R}^n\times \mathbb{R}^m\to\mathbb{R}^c$ with \(c\in\mathbb{N}\) denote a continuously differentiable function that defines the constraints on the joint state \(x_t\) and joint control input \(u_t\) at each time:
\vspace{-0.2cm}
\begin{equation}\label{eq:constraints}
    g(x_t, u_t)\leq {0}_c,
\end{equation}
\vspace{-0.6cm}

\noindent where ${0}_c$ is a vector of zeros in $\mathbb{R}^c$. The aim of the constraint function $g$ is to capture individual agents' strict preferences, such as each agent avoiding obstacles or avoiding collision with other agents. It should be noted that coupling amongst agents' decisions is captured through constraints.

% \begin{figure}
%     \centering
%     \includegraphics[width=0.8\linewidth,trim={0 0.8cm 0 1cm},clip]{figures/WAFR-2-filtering.pdf}
%     \vspace{-0.7cm}
%     \caption{When using MultiNash-PF on a difficult two agent trajectory game, we first employ an implicit particle filter to obtain coarse estimates of GNEs.}
%     % \vspace{-0.7cm}
%     \label{fig:filtering}
% \end{figure}

% \begin{figure*}[t]
%     \centering
%     % \includegraphics[scale=0.28]{figures/WAFR_2_2.pdf}
%     \includegraphics[width=\linewidth]{figures/two_agent_6.pdf}
%     \caption{MultiNash-PF identifies all the local GNE trajectories for a challenging two agent trajectory game where two agents swapped positions while avoiding an obstacle. In this scenario, the obstacle radius is greater than the collision avoidance radius of the two agents, leading to the discovery of 6 distinct Nash equilibria. MultiNash-PF effectively obtains coarse estimates of local GNEs through implicit particle filtering, which are then refined using IPOPT to obtain local GNEs.}
%     \label{fig:experiment-2}
%     % \vspace{-0.6cm}
% \end{figure*}
%We denote the set of feasible joint states and joint inputs as \(\mathbb{D}\), defined as follows 
% \begin{equation}
%     \mathbb{D}\coloneqq \{\begin{bmatrix} x^\top & u^\top\end{bmatrix}^\top| g(x, u)\leq {0}_c, x\in\mathbb{X}, u\in\mathbb{U}\}. 
% \end{equation}
% \subsection{Agent's Objective functions}\label{subsec: traj cost}

We assume that each agent \(i\) aims to minimize an objective function. 
Let \(Q^i\) and \( Q_{\tau}^i\) be positive semidefinite and \(R^i\) be positive definite matrices for all \(i\in\mathcal{N}\). We assume that each agent \(i\) seeks to minimize the following objective:
\vspace{-0.1cm}
\begin{equation}\label{eqn: traj cost}
     \norm{x_\tau^i-\hat{x}_\tau^i}_{Q^i_\tau}^2+\sum_{t=0}^{\tau-1} \norm{x_t^i-\hat{x}_t^i}_{Q^i}^2 + \sum_{t=0}^\tau\norm{u_t^i}_{R^i}^2 
\end{equation}
\vspace{-0.3cm}

\noindent where \(\hat{x}_t^i\in\mathbb{R}^{n_i}\) is the reference state for agent \(i\) at time \(t\). Typically, the reference states for each agent is a path of minimum cost they would follow in the absence of other agents in the environment, e.g., a straight line connecting the start and goal location. \maulik{Such a quadratic cost structure is a standard assumption in various motion planning domains \cite{askari2025model}.} Furthermore, we denote the joint reference state of all agents at time $t$ as $\hat{x}_t$.  
% We use $\blkdiag(A_1,\ldots,A_n)$ to denote the block-diagonal of the matrices $A_1,\ldots,A_n$.
We denote by $\blkdiag(A_1,\ldots,A_n)$ the block-diagonal matrix formed by placing matrices $A_1,\ldots,A_n$ on the diagonal with zeros elsewhere.
Let $a_{0:\tau}:= \{a_0,\ldots,a_\tau\}$ denote the collection of vectors $a_t$'s over the entire horizon $0 \leq t \leq \tau$. We use \(\{{x}_{0:\tau}^i,{u}_{0:\tau}^i\}_{i=1}^N\) to denote the joint trajectory of all \(N\) the agents at all times. Note that using the joint notation, we can equivalently write the joint trajectories of all agents as \(\{{x}_{0:\tau},{u}_{0:\tau}\}\).

\subsection{Local Generalized Nash Equilibrium Trajectories} 

Due to the interactive nature of the problem, generally, it is not possible for all agents to minimize their costs simultaneously while satisfying the constraints. Therefore, the interaction outcome is best captured by \emph{local GNE trajectories}. With dynamics as \eqref{eq:dynamics} and constraints as \eqref{eq:constraints}, we denote the game by 
\vspace{-0.1cm}
\begin{equation*}
    \mathcal{G} := \left( \mathcal{N}, \{Q^i\}_{i\in\mathcal{N}}, \{Q^i_\tau\}_{i\in\mathcal{N}},\{R^i\}_{i\in\mathcal{N}} , g,f, \hat{x}_{0:\tau} \right).
\end{equation*}
\vspace{-0.5cm}

\noindent For a given $\epsilon$, we define the set of local trajectories around a joint trajectory as
$\mathcal{D}\left( \{{x}_{0:\tau},{u}_{0:\tau}\}, \epsilon \right) =
    \left\{ \{{x}'_{0:\tau},{u}'_{0:\tau}\}  \bigg| \sum_{t=0}^\tau \left(\norm{x_t-{x}'_t}^2_2+ \norm{u_t-{u}'_t}^2_2\right)\leq \epsilon  \right\}.
$
A local GNE trajectory for $\mathcal{G}$ is then defined as follows.
\begin{definition}\label{def: Nash}
    % Let \(\{\{\overline{x}_t^i,\overline{u}_t^i\}_{t=0}^{\tau}\}_{i=1}^N\) denote joint trajectories of all \(N\) agents. 
    A set of joint trajectories \(\{{x}_{0:\tau}^{i^\star},{u}_{0:\tau}^{i^\star}\}_{i=1}^N\) is a local GNE trajectory for $\mathcal{G}$ if for every agent \(i\in\mathcal{N}\), the agent's trajectory \(\{{x}_{0:\tau}^{i^\star},{u}_{0:\tau}^{i^\star}\}\) is an optimal solution of the following optimization problem within a local neighborhood $\mathcal{D}\left( \{{x}^\star_{0:\tau},{u}^\star_{0:\tau}\}, \epsilon \right)$ for some $\epsilon > 0$:
\begin{align}\label{opt: local Nash}
      \min_{\{x_{0:\tau}^i, u_{0:\tau}^i\}} & \norm{x_\tau^i-\hat{x}_\tau^i}_{Q^i_\tau}^2+\sum_{t=0}^{\tau-1} \norm{x_t^i-\hat{x}_t^i}_{Q^i}^2 + \sum_{t=0}^\tau\norm{u_t^i}_{R^i}^2 \nonumber\\ \nonumber
     \mathrm{s.t} \quad & x^i_0=\hat{x}^i_0, \enskip x^i_{t+1} = f^i(x^i_t,u^i_t),\enskip 0\leq t\leq  \tau-1, \\ 
    & g(\tilde{x}_t,\tilde{u}_t) \leq {0}_c, \enskip 0\leq t\leq \tau, 
\end{align}
% \begin{equation}
%     \begin{array}{ll}
%       \underset{\{x_{0:\tau}^i, u_{0:\tau}^i\}}{\text{minimize}} & \norm{x_\tau^i-\hat{x}_\tau^i}_{Q^i_\tau}^2+\sum_{t=0}^{\tau-1} \norm{x_t^i-\hat{x}_t^i}_{Q^i}^2 \\ & \qquad \qquad + \sum_{t=0}^\tau\norm{u_t^i}_{R^i}^2\\
%      \mbox{subject to} & x_0^i=\hat{x}_0^i,\enskip {x}_{t+1}^i = f^i({x}_t^i,{u}_t^i), \enskip 0\leq t\leq \tau-1, \\
%     & g(\tilde{x}_t,\tilde{u}_t) \leq {0}_c, \enskip 0\leq t\leq \tau,\\
%     % & \sum_{t=0}^\tau \left(\norm{x_t^i-{x}_t^{i^\star}}^2_2+ \norm{u_t^i-{u}_t^{i^\star}}^2_2\right)\leq \epsilon,
% \end{array}
% \end{equation}
where we use the following notation in \(g(\tilde{x}_t,\tilde{u}_t)\)
\begin{equation*}
    \begin{aligned}
         \tilde{x}_t &\coloneqq \begin{bsmallmatrix}
        ({x}_t^{1^\star})^\top &  ({x}_t^{2^\star})^\top & 
        \cdots & ({x}_t^{{i-1}^\star})^\top & (x_t^i)^\top & ({x}_t^{{i+1}^\star})^\top & \cdots & ({x}_t^{N^\star})^\top 
    \end{bsmallmatrix}^\top ,\\
    \tilde{u}_t & \coloneqq \begin{bsmallmatrix}
        ({u}_t^{1^\star})^\top &  ({u}_t^{2^\star})^\top & 
        \cdots & ({u}_t^{{i-1}^\star})^\top & (u_t^i)^\top & ({u}_t^{{i+1}^\star})^\top & \cdots & ({u}_t^{N^\star})^\top 
    \end{bsmallmatrix}^\top.
    \end{aligned}
\end{equation*}
\end{definition}

Intuitively, at a local GNE, no agent can reduce their cost function by independently changing their trajectory to any alternative feasible trajectory within the local neighborhood of the equilibrium trajectories. It is important to highlight that solving~\eqref{opt: local Nash} requires solving a set of $N$ coupled constrained optimal control problems, which is often computationally expensive. Therefore, finding all the local GNEs of the game can be time-consuming when using traditional game-theoretic methods that aim to solve coupled optimal control problems.

However, an important property of our trajectory game is that it is a potential game, i.e., a GNE of the original game can be computed by optimizing a single constrained optimal control problem called \emph{Potential}, as described in \cite{bhatt2023efficient}.  This reduces the computational cost of computing the GNEs of the game. In the following, we will prove that game $\mathcal{G}$ can be expressed as a constrained trajectory potential game and that any local minimizer of the potential function is a local GNE.

\begin{theorem}\label{prop: equivalence}

A trajectory \(\{{x}_{0:\tau}^{i^\star},{u}_{0:\tau}^{i^\star}\}_{i=1}^N\) is a local GNE if within a local neighborhood $\mathcal{D}\left( \{{x}^\star_{0:\tau},{u}^\star_{0:\tau}\}, \epsilon \right)$, it is an optimal solution of
% Suppose that there exists \(\delta\in\mathbb{R}_{>0}\) such that \(\{({x}^\star_t, {u}^\star_t)\}_{t=0}^\tau\) is an optimal solution of the following optimization problem
\begin{align}\label{opt: traj local}
      \min_{\{x_{0:\tau},u_{0:\tau}\}} & \norm{x_\tau-\hat{x}_\tau}_{Q_\tau}^2+\sum_{t=0}^{\tau-1} \norm{x_t-\hat{x}_t}_{Q}^2 + \sum_{t=0}^\tau\norm{u_t}_{R}^2 \nonumber\\ \nonumber
     \mathrm{s.t} \quad & x_0=\hat{x}_0, \enskip x_{t+1} = f(x_t,u_t),\enskip 0\leq t\leq  \tau-1, \\
    & g(x_t,u_t) \leq {0}_c, \enskip 0\leq t\leq \tau,
\end{align}
for some \(\epsilon>0\), where $Q_\tau \coloneqq \blkdiag(Q^1_\tau,\ldots,Q^N_\tau),\enskip  Q \coloneqq \blkdiag(Q^1,\ldots,Q^N)$ and $R \coloneqq \blkdiag(R^1,\ldots,R^N)$. 
% with $\blkdiag$ denoting block-diagonal of the matrices.
% Let 
% \begin{equation}
%     \begin{aligned}
%         \overline{x}_t & \coloneqq \begin{bmatrix}
%         (\overline{x}_t^1)^\top & (\overline{x}_t^2)^\top &
%         \cdots & (\overline{x}_t^N)^\top 
%     \end{bmatrix}^\top ,\\ 
%     \overline{u}_t&\coloneqq \begin{bmatrix}
%         (\overline{u}_t^1)^\top & (\overline{u}_t^2)^\top &
%         \cdots & (\overline{u}_t^N)^\top 
%     \end{bmatrix}^\top ,
%     \end{aligned}
% \end{equation}
% be a partition such that \(\overline{x}_t^i\in\mathbb{R}^{n_i}\) and \(\overline{u}_t^i\in\mathbb{R}^{m_i}\) for all \(i=1, 2, \ldots, N\) and \(t=0, 1, \ldots, \tau\). Then \(\{\{\overline{x}_t^i, \overline{u}_t^i\}_{t=0}^\tau\}_{i=1}^N\) is a generalized local Nash equilibrium trajectory.
\end{theorem}

\begin{proof}
This theorem follows from \cite[Thm. 2]{bhatt2023efficient}.
%, which states that if each agent's cost function depends only on their states and actions, then the underlying game is a CPTG. 
Since agent costs from \eqref{eqn: traj cost} depend only on the individual agent's own states and control inputs, we use the results from \cite[Thm. 2]{bhatt2023efficient} to obtain the potential function of $\mathcal{G}$ to be
% \begin{equation}\label{eq: agent sum}
%     \begin{array}{ll}
%       \underset{\{x_t^i, u_t^i\}_{t=0}^\tau}{\text{minimize}} &  \sum_{i=1}^{N}\left(\frac{1}{2} \norm{x_\tau^i-\hat{x}_\tau^i}_{Q^i_\tau}^2+\frac{1}{2}\sum_{t=0}^{\tau-1} \norm{x_t^i-\hat{x}_t^i}_{Q^i}^2 +  \frac{1}{2}\sum_{t=0}^\tau\norm{u_t^i}_{R^i}^2\right)\\
%      \mbox{subject to} & x_0=\hat{x}_0,\enskip  x_{t+1} = f(x_t,u_t),\enskip 0\leq t\leq  \tau-1, \\
%     & g(x_t,u_t) \leq {0}_c, \enskip 0\leq t\leq \tau,\\
%     & \sum_{t=0}^\tau \left(\norm{x_t-\overline{x}_t}^2_2+ \norm{u_t-\overline{u}_t}^2_2\right)\leq \delta.
% \end{array}
% \end{equation}
\vspace{-0.1cm}
\begin{equation}\label{eq: agent sum}
    \sum_{i=1}^{N}\left(\norm{x_\tau^i-\hat{x}_\tau^i}_{Q^i_\tau}^2 + \sum_{t=0}^{\tau-1} \norm{x_t^i-\hat{x}_t^i}_{Q^i}^2 + \sum_{t=0}^\tau\norm{u_t^i}_{R^i}^2\right).
\end{equation}
Using the joint notations, we can re-write the sums in \eqref{eq: agent sum} as \eqref{opt: traj local}.
% $
% \sum_{i=1}^{N}\left( \norm{x_\tau^i-\hat{x}_\tau^i}_{Q^i_\tau}^2\right) = \norm{x_\tau-\hat{x}_\tau}_{Q_\tau}^2,
% $
% $\sum_{i=1}^{N}\left( \norm{x_t^i-\hat{x}_t^i}_{Q^i}^2 \right)= \norm{x_t-\hat{x}_t}_{Q}^2,$
% $\sum_{i=1}^{N}\left(\norm{u_t^i}_{R^i}^2\right) = \norm{u_t}_{R}^2$.
% % \begin{align*}
% %     \sum_{i=1}^{N}\left(\frac{1}{2} \norm{x_\tau^i-\hat{x}_\tau^i}_{Q^i_\tau}^2\right) & = \frac{1}{2} \norm{x_\tau-\hat{x}_\tau}_{Q_\tau}^2, \\
% %     \sum_{i=1}^{N}\left(\frac{1}{2} \norm{x_t^i-\hat{x}_t^i}_{Q^i}^2 \right) &= \frac{1}{2}\norm{x_t-\hat{x}_t}_{Q}^2, \nonumber \\
% %     \sum_{i=1}^{N}\left(\frac{1}{2}\norm{u_t^i}_{R^i}^2\right) & = \frac{1}{2}\norm{u_t}_{R}^2,
% % \end{align*}
% which in turn gives us \eqref{opt: traj local}.
Therefore, we have proven that $\mathcal{G}$ is a CPTG with potential function \eqref{opt: traj local}. Hence, using Theorem 2 from \cite{bhatt2023efficient}, any local solution of \eqref{opt: traj local} will also be a local GNE of $\mathcal{G}$.
\end{proof}
Using Theorem~\ref{prop: equivalence}, we want to find all local solutions of \eqref{opt: traj local} that will also be local GNE trajectories as they correspond to different interaction modes and outcomes. We would like to acknowledge that finding local solutions of \eqref{opt: traj local} will not necessarily give us all the GNEs of the original, as Theorem~\ref{prop: equivalence} provides only a sufficient condition for computing local GNEs. Furthermore, it is known that the problem of finding all Nash equilibria is a PPAD-complete problem even in the case of static games~\cite{daskalakis2009complexity}. However, for practical purposes in multi-agent navigation, finding local solutions of \eqref{opt: traj local} has proven to be enough for us.

The conventional approach to solving~\eqref{opt: traj local} relies on techniques such as interior point methods \cite{Wachter:2006wt} or sequential quadratic programming~\cite{boggs1995sequential}.
Naively, the local solutions can be obtained by exhaustive initialization of random trajectories to an optimization solver, such as IPOPT~\cite{Wachter:2006wt}, hoping to explore the solution space of local Nash equilibria. However, randomly exploring the nonconvex optimization landscape via gradient-based methods is computationally prohibitive and inefficient, as the same local equilibria can be arrived at from multiple random initializations. To address this, we utilize a Bayesian inference approach in the following section to efficiently recover all identified minima of~\eqref{opt: traj local}.

%% file: filtering.tex
\begin{figure}
    \centering
    \includegraphics[width=0.8\linewidth,trim={0 0.8cm 0 1cm},clip]{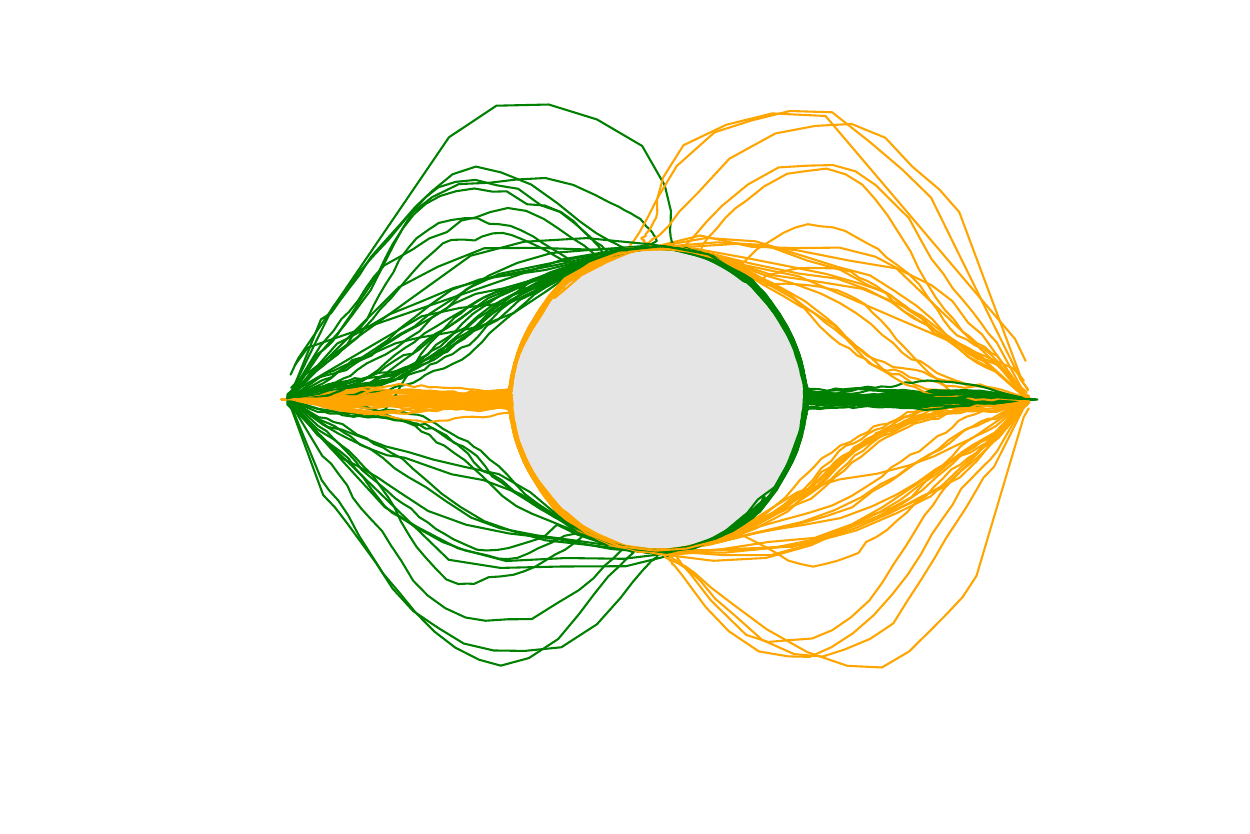}
    \vspace{-0.7cm}
    \caption{When using MultiNash-PF on a difficult two agent trajectory game, we first employ an implicit particle filter to obtain coarse estimates of GNEs.}
    % \vspace{-0.7cm}
    \label{fig:filtering}
\end{figure}
% \vspace{-0.3cm}
\section{Identifying Local Nash Equilibria via Bayesian Inference}\label{sec: NKF}
% \vspace{-0.3cm}
In this section, we propose our approach to solve~\eqref{opt: traj local} from a Bayesian inference perspective. We reformulate the optimal control problem~\eqref{opt: traj local} into an equivalent Bayesian inference problem by utilizing the reference trajectory and constraint violations as virtual measurements that provide evidence for the optimal inference of the system trajectory in~\eqref{opt: traj local}. 
To this end, we first reformulate problem~\eqref{opt: traj local} by considering soft-constraints and introducing auxiliary variables $y_{x,t}$ and $y_{g,t}$ as:
\vspace{-0.3cm}
    \begin{align}\label{moving horizon estimation}
       \min_{\{v_t, w_t, \eta_t\}_{t=0}^\tau} & 
  \norm{v_\tau}_{Q_\tau}^2 + \sum_{t=0}^{\tau-1} \norm{v_t}_{Q}^2 + \sum_{t=0}^\tau \left( \norm{w_t}_{R}^2 + \norm{\eta_t}^2_{Q_\eta} \right), \nonumber \\  
    \mathrm{s.t} \quad & \begin{bmatrix}
        x_{t+1}\\
        u_{t+1}
    \end{bmatrix}=\begin{bmatrix}
        f(x_t, u_t)\\
        0_{m}
    \end{bmatrix}+\begin{bmatrix}
        0_{n}\\
        w_t
    \end{bmatrix}, \nonumber \\
    & \begin{bmatrix}
        y_{x,t}\\
         y_{g,t}
    \end{bmatrix}=\begin{bmatrix}
        x_t\\
        \psi\left(g(x_t, u_t)\right)
    \end{bmatrix}+\begin{bmatrix}
        v_t\\
        \eta_t
    \end{bmatrix}, \nonumber \\
    & \begin{bmatrix}
        x_0\\
        u_0
    \end{bmatrix} =\begin{bmatrix}
        \hat{x}_0\\
        w_{t-1}
    \end{bmatrix}, \quad 0 \leq t \leq \tau, 
\end{align}
where $v_t$, $w_t$, and $\eta_t$ are bounded disturbances. The weight matrix $Q_{\eta} \in \mathbb{R}^{n_c \times n_c} $ quantifies the slackness in  satisfying constraints. Further, the degree of constraint violation is captured by a barrier function $\psi(\cdot, \cdot)$, which is defined as
\begin{equation}\label{solftplus}
    \psi\left(g(x_t, u_t)\right) = \frac{1}{\alpha}\ln\left(1 + \exp\left(g\left(x_t, u_t\right)\right) \right),
\end{equation}
where $\alpha$ is a parameter that tunes the strictness of constraint enforcement. 

% We note that the problem in~\eqref{moving horizon estimation} resembles a moving horizon estimation problem~\cite{Rao:TAC:2003}. 
A direct attempt at solving~\eqref{moving horizon estimation} using traditional gradient-based methods poses a similar computational burden in identifying all of the local equilibria compared to problem~\eqref{opt: traj local}. However, the above connection between~\eqref{opt: traj local} and the moving horizon estimation problem~\eqref{moving horizon estimation} highlights the possibility to view the problem from a Bayesian inference perspective. To this end, we construct the following virtual state-space model from~\eqref{moving horizon estimation} as
\begin{align}\label{Virtual-System-Compact}
\bar{x}_{t+1} = \bar{f}(\bar{x}_t) + \bar{w}_t, \;
{\bar y}_t =\bar h(  \bar{x}_t ) +\bar v_t,
\end{align}
where 
\vspace{-0.3cm}
\begin{align*}
\bar{x}_t&=\begin{bmatrix}x_{t} \\ u_{t}\end{bmatrix}, \: \bar{y}_t= \begin{bmatrix} y_{x,t} \\ y_{g,t}\end{bmatrix}, \: \bar{w}_t=\begin{bmatrix} 0_n \\ w_t\end{bmatrix}, \\ \bar{v}_t =\begin{bmatrix}v_t \\ \eta_t\end{bmatrix}, \; \bar f(&\bar x_t)=\begin{bmatrix}f(x_t, u_t) \\ {0}_m\end{bmatrix},  \; \bar h(\bar x_t)=\begin{bmatrix}x_t \\ \psi(g\left( x_t, u_t\right))\end{bmatrix}, 
\end{align*}
where we relax $\bar w_t$ and $\bar v_t$ for $0\leq t \leq \tau$ to be stochastic disturbances modeled as Gaussian-distributed random variables with density
\begin{equation*}\label{noise pdfs}
    \bar v_t \sim \mathcal{N}(0_{n+c}, \bar Q), \; \bar v_{\tau} \sim \mathcal{N}(0_{n+c}, \bar Q_{\tau}), \; \bar w_t \sim \mathcal{N}(0_{n+m}, \bar R),
\end{equation*} 
with $\bar Q \coloneqq \blkdiag(Q^{-1}, Q_\eta^{-1})$, $\bar Q_\tau \coloneqq \blkdiag(Q_{\tau}^{-1}, Q_\eta^{-1})$, $\bar R \coloneqq  \blkdiag(0_{n\times n}, R^{-1})$. The state estimation relevant to~\eqref{Virtual-System-Compact} is to estimate the state $\bar x_t$ given the measurements $\bar y_t$. The value of the virtual measurements must be set in a way that steers the inference problem to have the same optima as in problem~\eqref{opt: traj local}. We achieve this by noting that $y_{x,t}$ is a measurement value used to correct the state estimate of $x_t$ towards the reference $\hat x_t$ while $y_{g,t}$ is to satisfy constraints through~\eqref{solftplus}. Observing that the barrier function in~\eqref{solftplus} outputs zero when the constraints are satisfied, we set the virtual measurement for the barrier function to be equal to $0_{c}$. Hence, the virtual measurement that will drive $\bar x_t$ towards optimality takes the value
\vspace{-0.1cm}
\begin{equation*}
    \bar y_t = \begin{bmatrix}
         \hat x_t \\ {0}_c
    \end{bmatrix}, 0\leq t \leq \tau.
\end{equation*}

The Bayesian inference problem pertaining to~\eqref{Virtual-System-Compact} is one that characterize the conditional distribution $\p(\bar{x}_{0:\tau} | \bar{y}_{0:\tau})$. In the following lemma, we show that an equivalence holds between a posteriori estimate of $\p(\bar{x}_{0:\tau} | \bar{y}_{0:\tau})$ and~\eqref{opt: traj local}.
\begin{lemma}\label{MAP lemma}
Assuming independence between the noise vectors $\bar v_t$ and $\bar w_t$ in~\eqref{Virtual-System-Compact}, the problem in~\eqref{opt: traj local} has the same optima as
\vspace{-0.4cm}
\begin{align*}\label{MAP problem}
\bar x_{0:\tau}^{\star} = \arg \max_{\bar{x}_{0:\tau}} \: \log\p(\bar x_{0:\tau} | \bar{y}_{0\tau}). 
\end{align*}
\end{lemma}
\noindent We refer the interested reader to \cite[Thm. 1]{askari2023model} for the proof. The above lemma implies that the local GNEs in~\eqref{opt: traj local} translate to multiple local maxima of \(\p(\bar x_{0:\tau} | \bar{y}_{0:\tau})\). Hence, the Bayesian inference problem for~\eqref{Virtual-System-Compact} involves characterizing the multi-modal PDF \(\p(\bar x_{0:\tau} | \bar{y}_{0:\tau})\). Generally, state estimation of a multi-modal PDF does not admit closed-form solutions~\cite{Sarkka:Cambridge:2013}. Hence, we resort to approximate solutions. A powerful approach is to use Monte Carlo methods that empirically approximate the distribution via a set of $J$ particles as 
\vspace{-0.4cm}
\begin{equation}\label{MonteCarlo Approx}
   \p(\bar x_{0:\tau} | \bar y_{0:\tau}) \approx \sum_{j = 1}^{J}w_{\tau}^j\delta\left(x_{0:\tau} - x_{0:\tau}^j\right),
\end{equation}
\vspace{-0.3cm}

\noindent where $\delta(\cdot)$ is the Dirac delta function and $w_{\tau}^i$ is the weight assigned to a sample trajectory $\bar{x}_{0:\tau}^i$ for $j = 1,\hdots, J$. \maulik{It is important to note that, in our setting, each particle represents a complete trajectory.}
% After identifying the local GNE through the above approximation, we extract the trajectories that appear as local GNE and use them to initialize an optimization solver. This completes the exposition of the \emph{MultiNash-PF} framework.

\begin{figure*}[t]
    \centering
    \includegraphics[width=\linewidth]{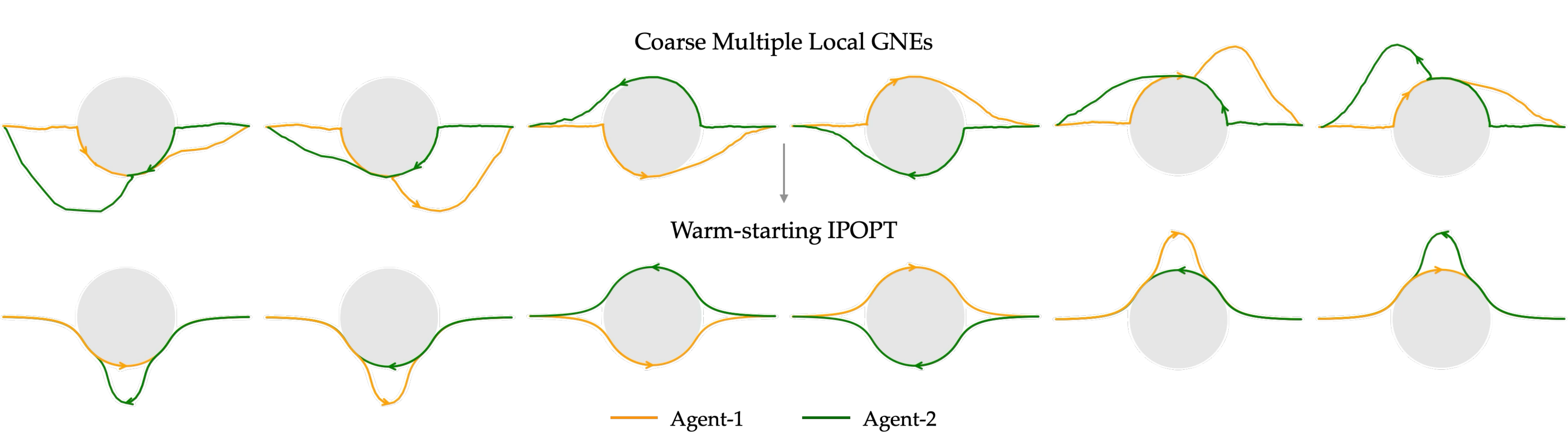}
    \caption{MultiNash-PF identifies all the local GNE trajectories for a challenging two-agent trajectory game where two agents swapped positions while avoiding an obstacle. In this scenario, the obstacle radius is greater than the collision avoidance radius of the two agents, leading to the discovery of 6 distinct Nash equilibria. MultiNash-PF effectively obtains coarse estimates of local GNEs through implicit particle filtering, which are then refined using IPOPT to obtain local GNEs.}
    \label{fig:experiment-2}
    \vspace{-0.6cm}
\end{figure*}

% In the following, we will briefly introduce the implicit particle filter method. 
To achieve the approximation in~\eqref{MonteCarlo Approx}, we employ the implicit particle filter in~\cite{Askari:2022:Auto} due to the twofold benefit it holds. First, as suggested in Lemma~\ref{MAP lemma}, the local maxima (i.e., high-probability regions) of \(\p(\bar x_{0:\tau} | \bar y_{0:\tau})\) correspond to the local GNEs. Hence, focusing our sampling efforts towards these high-probability regions can reduce the number of required samples for the approximation in~\eqref{MonteCarlo Approx}. This approach aligns with the principle of implicit importance sampling, whose objective is to build a mapping that draws particles from the high-probability regions of the filtering distribution in~\eqref{MonteCarlo Approx}. The implicit particle filter mitigates the particle degeneracy problem that pertains to particle filters by focusing the samples on the high-probability regions, which results in samples with large enough weights that are less prone to degeneration. Therefore, the implicit particle filter can often be deployed without requiring a resampling step~\cite{Askari:2022:Auto}, thereby resulting in smoother approximated trajectories, although occasional resampling may still be required.
% Second, we want to avoid the traditionally needed resampling step in particle filtering methods to ensure that the inferred trajectories are smooth. 

We leverage the Markovian property of the virtual system~\eqref{Virtual-System-Compact} to recursively approximate~\eqref{MonteCarlo Approx} using Bayes' rule:
\vspace{-0.1cm}
\begin{equation}\label{Bayes Recursion}
    \p(\bar{x}_{0:t} | \bar{y}_{0:t}) = \p(\bar{y}_t | \bar{x}_t)\p(\bar {x}_t | \bar{x}_{t-1})\p(\bar{x}_{0:t-1} | \bar{y}_{0:t-1}),
\end{equation}
for $0 \leq t \leq \tau$. This indicates that given the particle $\bar{x}_{t-1}^j$ from the prior distribution $\p(\bar{x}_{0:t-1} | \bar{y}_{0:t-1})$, we can obtain $\bar{x}_t^j$. As shown in~\cite{Askari:2022:Auto}, we use implicit importance sampling and identify an implicit mapping to sample from high probability regions of~\eqref{Bayes Recursion}. The implicit map is constructed by considering sampling from a reference random vector $\xi_t$ with density $p(\gamma_t)$ and mapping it to high probability regions of $\p(x_{0:\tau} | y_{0:\tau})$, such that a sample $\gamma^i_t$ is mapped to $x_t^i$. The mapping connects the highly probable regions of $p(\xi)$ to $\p(x_{0:\tau} | y_{0:\tau})$ by letting 
\begin{equation}\label{implicit mapping}
F(x_t^i) - \min F(x_t^i) = G(\gamma_t^i) - \min G(\gamma_t^i),
\end{equation}
where $F(x_t^i) = -\log\left( \p(x_t^i | y_{1:t})\right)$ and $G(\gamma_t^i) = -\log\left( p(\gamma_t^i)\right)$. Finding an explicit solution to~\eqref{implicit mapping} is computationally intractable due to the involved nonconvex optimization. To bypass this computational challenge, we employ a local Gaussian approximation around $\bar{x}_{t-1}^j$ such that 
\begin{equation*}\label{local Gaussian Approx}
 p(\bar{x}_t^j| \bar{y}_{0:t})\sim \mathcal{N}\left(\hat m_t^j, \hat \Sigma_t^{x,j}\right),
\end{equation*}
\vspace{-0.4cm}

\noindent where $\hat{m}_t^j$ and $\hat \Sigma_t^{x,j}$ are the mean and covariance of $\bar{x}_t^j$, respectively. These quantities can be approximated recursively using an Unscented Kalman filter (UKF)~\cite{wan2000unscented} for all $j = 1,\ldots, N$. Then, the implicit map can be computed in closed form as
\vspace{-0.3cm}
\begin{equation*}
    \bar x_t^j = \hat m_t^j + \sqrt{\hat \Sigma_t^{x,j}}\gamma_t^j,
\end{equation*}
\vspace{-0.6cm}

\noindent where $\gamma_t^j$ is a sample from $p(\gamma_t)\sim \mathcal{N}(0, I_n)$ and $I_n$ denotes the identity matrix of size $n$.

% If necessary, we perform resampling for the particle trajectory $\bar{x}_{0:t}^i$.
We repeat this process until $t= \tau$ for each particle to obtain the set of trajectories $\{\bar{x}_{0:\tau}\}_{j=1}^J$, which will include the coarse estimates of local GNEs. This completes the unscented implicit particle filter, which identifies the coarse estimates of multiple local Nash equilibria in~\eqref{opt: traj local}.

In the last step, we extract local equilibria from $\{x_{0:\tau}\}_{j=1}^J$. As viewed in Fig.~\ref{fig:filtering}, the particle set contains multiple trajectories for each GNE. Therefore, to isolate the GNE trajectories, we use a clustering method to aggregate the particles close to each other in one cluster. We leverage methods such as hierarchical clustering~\cite{murtagh2012algorithms} or DBSCAN~\cite{ester1996density} for clustering purposes. It should be noted that clustering methods require a metric to compute distances between the data points. We use Fréchet distance~\cite{eiter1994computing} as our metric to compute distances between the trajectories. The clustering of trajectories will result in one cluster for each GNE.
For each cluster, we compute the average trajectory which serves as an initial guess for the IPOPT solver. The solver then refines this initial guess by minimizing \eqref{opt: traj local}, ultimately converging to a precise local GNE. 
% This refinement step is necessary because the averaged trajectories, while close to the GNEs, may not exactly satisfy the game's equilibrium conditions.
% We take the average of trajectories over each cluster and feed them as a warm start to the IPOPT solver to obtain the local GNEs.
The algorithmic steps of the proposed \emph{MultiNash-PF} are provided in Algorithm~\ref{alg: filter}.
% \vspace{-0.3cm}
\begin{algorithm}[h]
\caption{MultiNash-PF: Multiple local GNEs based on particle filter}
\label{alg: filter}
\begin{algorithmic}[1]
\State Formulate the constrained dynamic potential game~\eqref{opt: traj local}
\State Setup the virtual model~\eqref{Virtual-System-Compact}
\Statex Perform unscented implicit particle filtering as follows:
\State \parbox[t]{\dimexpr\linewidth-\algorithmicindent}{ 
\hspace{1em}Initialize the particles: $\bar{x}_0^j$, $\hat \Sigma^{x,j}_0$, and $w_0^j$ for 
\Statex \hspace{1em}  $j = 1,\hdots, J$}
\vspace{0.3em}
\State \hspace{1em} Employ Algorithm-1 from \cite{Askari:2022:Auto} combined with UKF 
\Statex \hspace{1em} to obtain coarse solutions, \(\{ x_{0:\tau}\}_{j=1}^J\).
\State Extract trajectories through clustering from \(\{ x_{0:\tau}\}_{j=1}^J\) and warm start IPOPT solver to obtain the local GNEs \(\{{x}_{0:\tau}^{i^\star},{u}_{0:\tau}^{i^\star}\}_{i=1}^N\) of $\mathcal{G}$.
\end{algorithmic}
\end{algorithm}
% \vspace{-0.3cm}

%% file: numerical_simulations.tex
% \vspace{-0.3cm}
\section{Numerical Simulations}\label{sec:sim}

% \subsection{Experiment-1}

To showcase the capabilities of our method, we take motivation from real-life highway scenarios. We first consider an interaction between two agents moving toward each other as shown in Fig~\ref{fig: two_unicycle}. We consider the following discrete-time unicycle dynamics to model each agent $i \in \{1,2\}$:
\vspace{-0.2cm}
\begin{align*}\label{unicycle}
    p^i_{t+1} & = p^i_t + \Delta t\cdot \nu^i_t\cos{\theta^i_t}, \;
    q^i_{t+1} = q^i_t + \Delta t\cdot \nu^i_t\sin{\theta^i_t} \nonumber \\
    \theta^{i}_{t+1} & = \theta^i_t + \Delta t\cdot\omega_i, \; \nu^i_{t+1} = \nu^i_{t} + \Delta \nu^i_{t}, \; \omega^i_{t+1} = \omega^i_{t} + \Delta \omega^i_{t},
\end{align*}
\vspace{-0.5cm}

\noindent where $\Delta t$ is the time-step size, $p^i_t$ and $q^i_t$ are $x$ and $y$ coordinates of the positions, $\theta^i_t$ is the heading angle from positive x-axis, $\nu^i_t$ is the linear velocity, and $\omega^i_t$ is the angular velocity of the agent $i$ at time $t$. Further, $\Delta \nu^i_t$ is the change in linear velocity, and $\Delta\omega^i_t$ is the change in angular velocity of the agent $i$ at time $t$, respectively. For each agent $i$, we denote the state vector $x^i_t$ and the action vector $u^i_t$ as follows
\vspace{-0.3cm}
\begin{equation*}
    x^i_t=\begin{bmatrix}
        p^i_t &
        q^i_t &
        \theta^i_t &
        \nu^i_t & 
        \omega^i_t 
    \end{bmatrix}^\top, \enskip u^i_t = \begin{bmatrix}
        \Delta \nu^i_t &
        \Delta \omega^i_t
    \end{bmatrix}^\top.
\end{equation*}
\vspace{-0.5cm}

\noindent We formulate this scenario as a two-player game according to cost functions as described in \eqref{eqn: traj cost}. The reference trajectory for both agents is given as a straight line on the highway, starting from their initial position to their respective goal location. With $\hat{Q} = \blkdiag(50,10,5,5,2)$, the reference tracking and control effort cost matrices are set as 
% $Q^1 = Q^2 = 0.6\cdot diag(50,10,5,5,2)$, $Q^1_\tau = Q^2_\tau = 100\cdot diag(50,10,5,5,2)$, and $R^1 = R^2 = diag(8,4)$.
\vspace{-0.1cm}
\begin{align}\label{cost}
    Q^i = 0.6\hat{Q}, \;
    Q^i_\tau = 100\hat{Q}, \;
    R^i = diag(8,4).
\end{align}
% \vspace{-0.2cm}
\noindent for $i\in\{1,2\}$. In addition, the agents are subject to the constraint of avoiding collision with each other, which we define as
\vspace{-0.3cm}
\begin{align}\label{collision_constraint}
     g_c(x_t,u_t) \coloneq -d(x^1_t,x^2_t) + r_{\text{col}} \leq 0,
\end{align}
where $d(x^1_t,x^2_t)$ is the Euclidean distance between the two agents and $r_{\text{col}} \coloneq 3m$ is the collision avoidance radius between the two agents. 

With this setup, we employ MultiNash-PF from Algorithm~\ref{alg: filter} to efficiently capture the multi-modality in the solution space and then extract the solutions using hierarchical clustering~\cite{murtagh2012algorithms} to provide a warm start trajectory to IPOPT~\cite{Wachter:2006wt} solver to recover different local GNEs. We use the Julia programming language and utilize the JuMP~\cite{dunning2017jump} library to solve the game using the IPOPT solver. The results of our method are provided in Fig~\ref{fig: two_unicycle}. 
% As we can see, both agents would go in a straight line to their goal location in the absence of the other agent. 
As we can see, the presence of two agents and collision avoidance constraints induces two different possible GNEs. Either both agents will yield to their right, or they will yield to their left to avoid collision. As shown in Fig~\ref{fig: two_unicycle}, our method can recover both equilibria.

\begin{figure}[t]
    \centering
    \includegraphics[width=\linewidth]{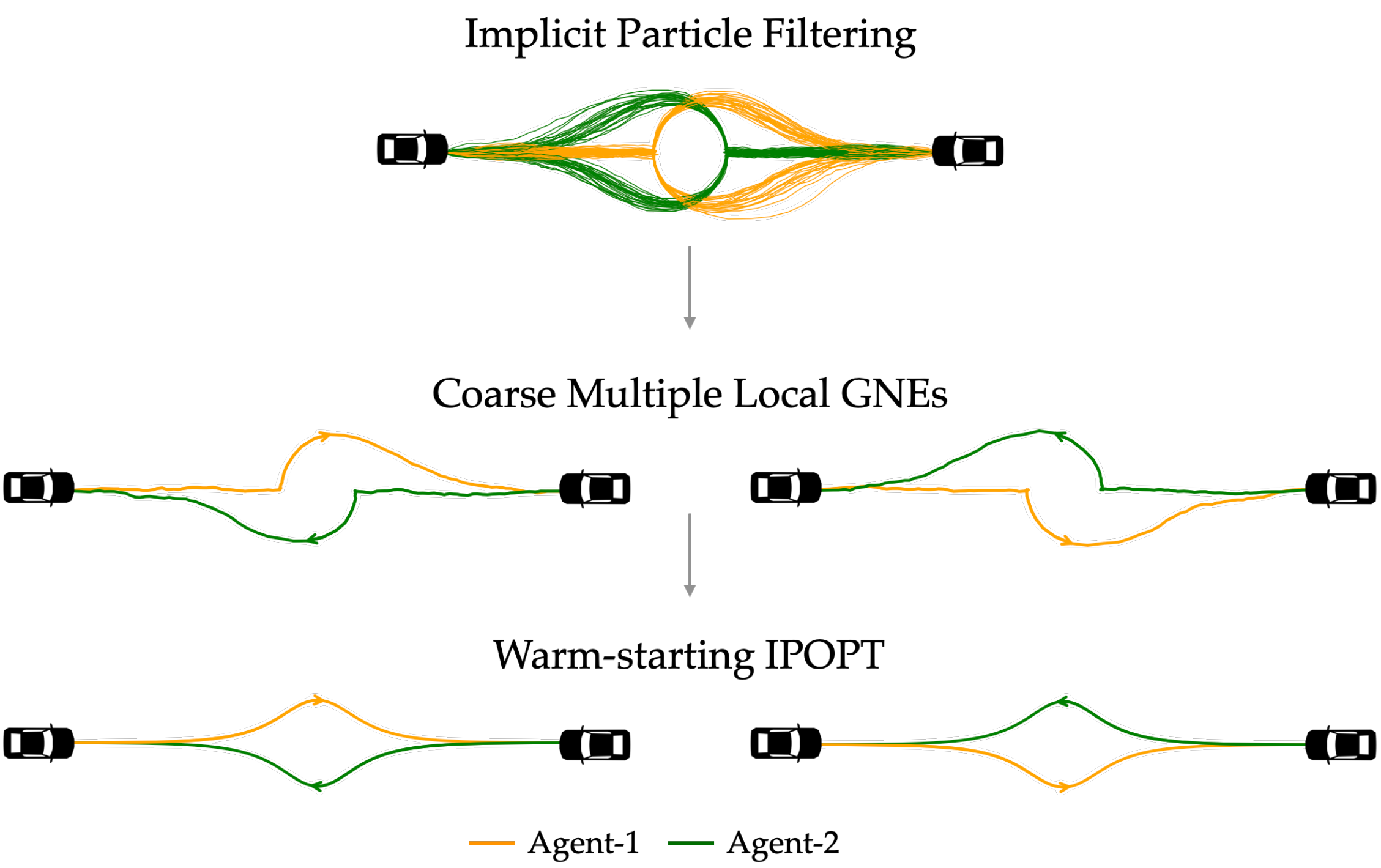}
    \caption{The Nash equilibrium trajectories found by \emph{MultiNash-PF} when two unicycle agents change their positions. First, we use the implicit particle filter to discover the coarse estimates of two different equilibria, and then utilize them as a warm start for IPOPT to obtain the two Nash equilibrium trajectories.}
    \label{fig: two_unicycle}
    \vspace{-0.7cm}
\end{figure}

% \begin{figure}
%     \centering
%     \includegraphics[trim={0.1cm 0.1cm 0.1cm 1cm},clip,width=10cm]{WAFR_2024/figures/WAFR_2_2.pdf}
%     \caption{The Nash equilibrium trajectories discovered by MultiNash-PF when two unicycle agents exchange their positions while avoiding an obstacle. Since the obstacle radius is larger than the collision avoidance radius of the two agents, 6 different Nash equilibria can exist. As shown in the figure, MultiNash-PF finds all the different equilibria efficiently.}
%     \label{fig:experiment-2}
% \end{figure}

% \subsection{Experiment-2}

Next, we consider a more complicated scenario in which two agents aim to exchange their positions while avoiding an obstacle of radius $r_{\text{obs}} = 4m$ located at the midpoint of the two agents. Note that the radius of the obstacle is larger than the collision avoidance radius of the two agents. 
% This setup induces multitudes of local equilibria as apart from going through opposite sides of the obstacle, both agents can go on the same side of the obstacle, and one of the agents can yield to the other one (see Fig~\ref{fig:experiment-2}).

We consider the same setup as before. In this example, in addition to \eqref{collision_constraint}, the constraints for this game will include both agents avoiding the obstacle $g^i_o(x_t,u_t):= \begin{bmatrix}
    -d(x^i_t,x_{\text{obs}}) + r_{\text{obs}}
\end{bmatrix}$ for $i=1,2$ where $x_{\text{obs}}$ is the location of the obstacle. We also assume that we have constraints on the control bounds $g_b(x_t,u_t) = \begin{bmatrix}
    -\nu^1_t, &
        -\nu^2_t, &
        |u^1_t| - u^1_b, &
        |u^2_t| - u^2_b
\end{bmatrix}$, where $u^1_b$ and $u^2_b$ are controlled action bounds for agents 1 and 2, respectively. We write the combined constraints as $g(x_t,u_t) := \begin{bmatrix}
        g_c & g^1_o & g^2_o & g_b
    \end{bmatrix}^\top \leq 0.$
% \begin{equation}
%     g(x_t,u_t) := \begin{bmatrix}
%         -d(x^1_t,x^2_t) + r_{\text{col}} \\
%         -d(x^1_t,x_{\text{obs}}) + r_{\text{obs}} \\
%         -d(x^2_t,x_{\text{obs}}) + r_{\text{obs}} \\
%         -v^1_t \\
%         -v^2_t \\
%         |u^1_t| - u^1_b\\
%         |u^2_t| - u^2_b
%     \end{bmatrix} \leq \mathbf{0}_9,
% \end{equation}
We choose the values of control input bounds to be $u^1_b = u^2_b = [0.15,0.75]^\top$.
% For such a scenario, multitudes of Nash equilibria may exist since the obstacle radius is larger than the inter-agent collision avoidance radius.
This setup induces multitudes of local equilibria.
Two trivial local equilibria are when the two agents pass each other through opposite sides of the obstacle. However, there exist four other nontrivial equilibria when both agents can go on the same side of the obstacle, and one of the agents can yield to the other one. It is challenging to recover all six equilibria without utilizing the capabilities of the implicit particle filter. \maulik{We fixed $J = 50$ particles and obtained the solution trajectories using implicit particle filtering from Algorithm~\ref{alg: filter} described in Section~\ref{sec: NKF}.} Upon performing simulation experiments with our method, we observe that MultiNash-PF can recover all six Nash equilibria as shown in Fig.~\ref{fig:filtering} and Fig.~\ref{fig:experiment-2}. Please note that, while trajectories may overlap spatially, collisions are avoided because agents traverse those regions at different times, as enforced by the constraints in our optimization.

\begin{figure}
    \centering
    \begin{subfigure}{\linewidth}
        \centering
        \includegraphics[height=2in]{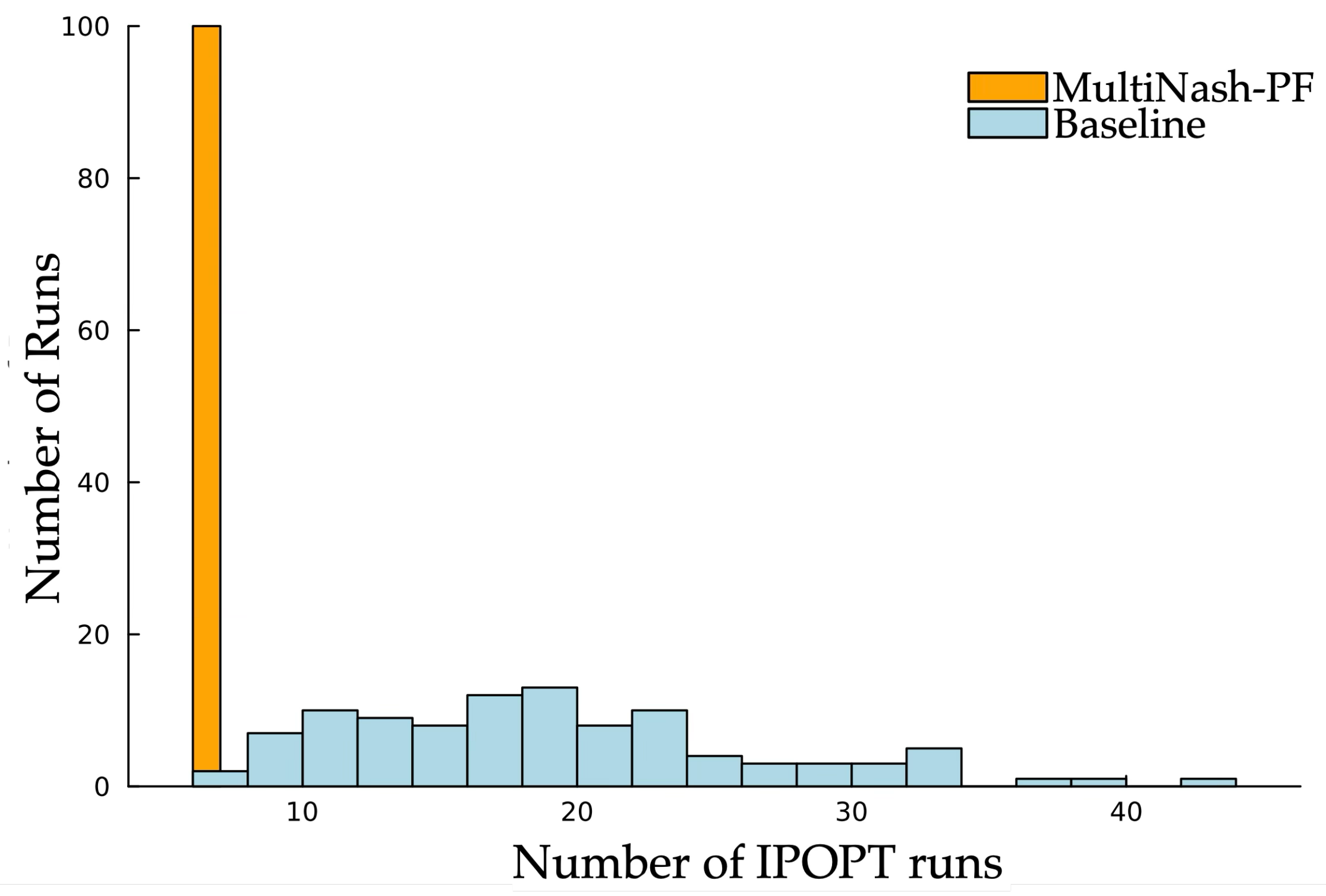}
        \caption{Histogram of the number of IPOPT runs required to obtain all six equilibria.}
        \label{fig:ipopt_runs}
    \end{subfigure}
    \vspace{0cm} % Add some vertical space between the subfigures
    \begin{subfigure}{\linewidth}
        \centering
        \includegraphics[height=2in]{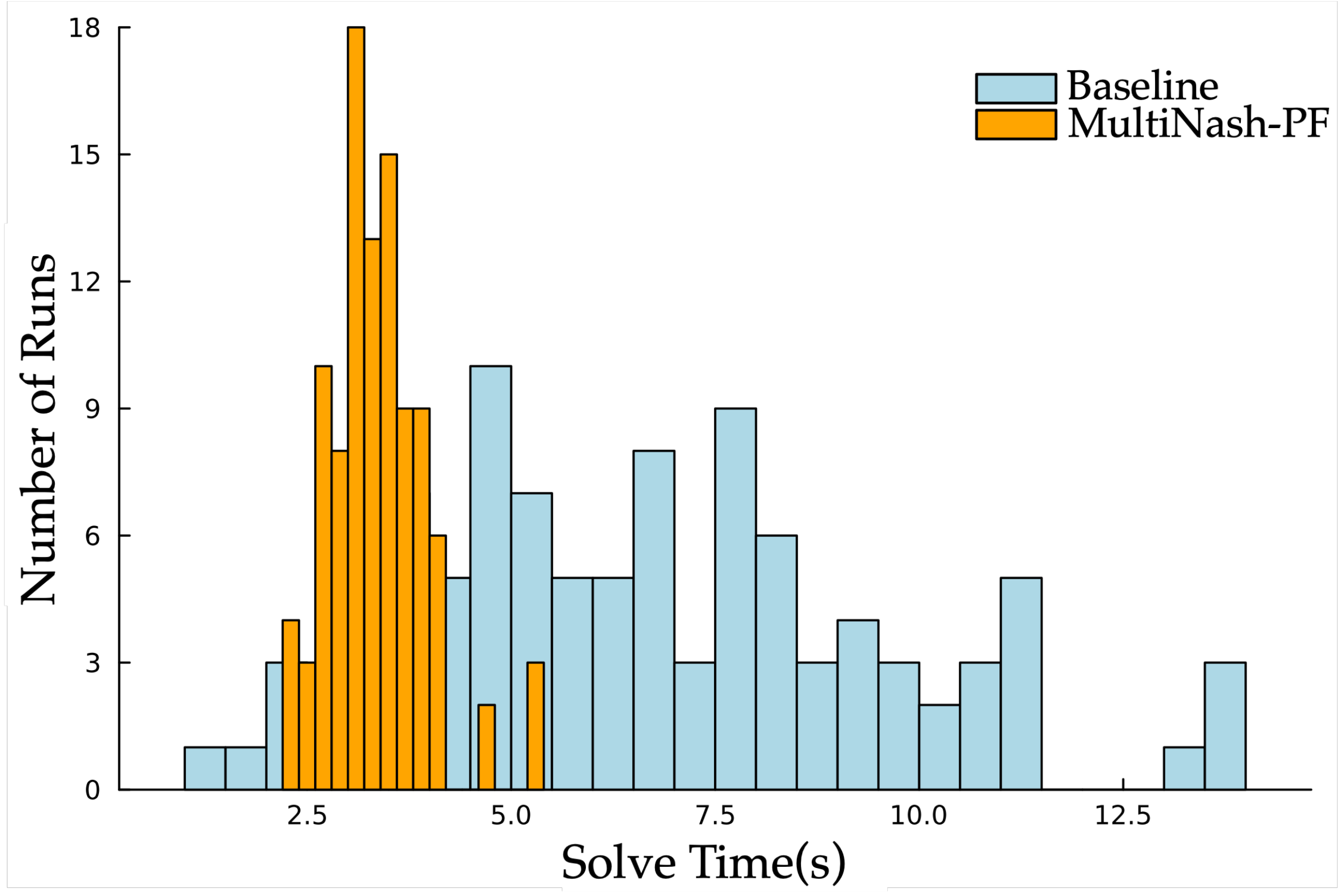}
        \caption{Histogram of the solve time comparison of MultiNash-PF and the baseline.}
        \label{fig:solve_time}
    \end{subfigure}
    \vspace{-0.6cm}
    
    \caption{Comparison of MultiNash-PF with the baseline for the experiment shown in Fig.~\ref{fig:experiment-2}. \\ 
        (a) For the baseline, the average number of IPOPT runs required to obtain all six different equilibria is $18.64 \pm 7.57$, while for MultiNash-PF, this is always $6$. MultiNash-PF requires 3 times fewer IPOPT runs.
        (b) MultiNash-PF reduces the solve time by about $50\%$ compared to the baseline while giving a much lower variance on the solve time.}
        \vspace{-0.7cm}
\end{figure}

\subsubsection{Comparison with Baseline}
In order to compare the effectiveness of our method, we consider a baseline method. For the baseline, we directly initialize the IPOPT solver with random perturbations of reference trajectories and compute the resulting optimal solutions. 
% For our method, we fixed $J = 50$ particles and obtained the solution trajectories using implicit particle filtering from Algorithm~\ref{alg: filter} described in Section~\ref{sec: NKF}. After that, we extract all the different equilibria and provide them as a warm starting point for IPOPT to obtain the final solutions. 
The computation time of our method is a combination of the time used in implicit particle filtering and the time used by IPOPT to obtain all the different equilibria given warm starting from filtering. 

For the baseline, we randomly initialize the IPOPT solver until we recover all 6 different equilibria. We repeat this experiment for 100 different Monte Carlo runs. We observe that it requires, on average, $18.64 \pm 7.57$ random initializations of IPOPT to obtain all six different equilibria using the baseline approach. However, with our method, the number of IPOPT runs required is always six because we always obtain all six modes of equilibria from implicit particle filtering. A histogram of the number of IPOPT runs required for obtaining all six equilibria is plotted in Fig.~\ref{fig:ipopt_runs}.

Furthermore, the solve time to recover all 6 different Nash equilibria for the baseline is $6.71 \pm 2.83 s$, while for MultiNash-PF, the solve time is $3.36 \pm 0.59 s$. Our method results in a solve-time reduction of $50\%$ compared to the baseline while giving a much lower variance on the solve time. A histogram of the solve times of the baseline compared to our method is plotted in Fig.~\ref{fig:solve_time}. As explained earlier, our method's solve time is a combination of the time used in implicit particle filtering and the time used by IPOPT. We observe that the time required for implicit particle filtering over 100 Monte Carlo runs is $0.37 \pm 0.072 s$ while the time used by IPOPT to obtain the final solutions is $2.99 \pm 0.57 s$. Therefore, in our MultiNash-PF, the time used by implicit particle filtering is $11\%$ of the total solve time, which shows that filtering is very efficient in exploring the solution space of local Nash equilibria.

% \begin{figure*}[th]
%     \centering
%     \includegraphics[scale = 0.25]{ieeeconf/figures/three_unicycle.pdf}
%     \caption{Description}
%     \label{fig:three_unicycle}
% \end{figure*}

%% file: experiment.tex
\section{Hardware Experiment}

Finally, we conduct a hardware experiment where a human and a robot exchange their positions as shown in Fig.~\ref{fig:experiment}. We perform these experiments using a TurtleBot 4, which is a ground robot with unicycle dynamics. We assume that the human also follows unicycle dynamics. The collision avoidance constraint between the TurtleBot and the human is assumed to be the same as that of \eqref{collision_constraint}. We use the Vicon motion capture system to measure the current position of both the robot and the human.

Similar as shown in Fig.~\ref{fig: two_unicycle}, two GNEs may exist, which will correspond to two different modes of interactions. Namely, \emph{yield-right equilibrium} when the human and robot yield to their right and \emph{yield-left equilibrium} when they yield to their left. For the robot to safely navigate around the human without collision, the robot needs to be aware of both the GNEs and needs to be able to determine which GNE the human chooses. Using MultiNash-PF, the robot first uses an implicit particle filter to precompute the coarse estimates of the two different equilibria and then utilizes them as a warm start for IPOPT to obtain two modes of interaction. We assume that the interaction happens over $T = 10s$. Then, similar to \cite{bhatt2023efficient}, the robot uses a potential function based game-theoretic planner in a model-predictive fashion with a planning horizon of 5s and $\Delta t = 0.1s$ to plan its motion.

The robot faces two \maulik{precomputed} choices of GNEs to follow recovered from the MultiNash-PF. At any given time $t$ of planning, the robot computes the Fréchet distance~\cite{eiter1994computing} between the trajectory exhibited by the human so far and the human's trajectory in both the computed GNEs. Let the distances of the human trajectory from both the GNEs be denoted by $d_1$ and $d_2$, respectively. The robot keeps an uncertainty on the GNE that the human chooses until $|d_1-d_2| > d^*$ where $d^*$ is a pre-specified threshold. For example, in Fig.~\ref{fig:experiment}, at $t = 2$, the values of $d_1$ and $d_2$ are almost identical, which indicates that the human has not chosen their mode of interaction. Whenever $|d_1-d_2|>d^*$ occurs, the robot becomes certain of the mode of interaction followed by the human. Then, the robot follows the GNE corresponding to $\arg\min_{i\in\{1,2\}}\{d_i\}$. For example, in Fig.~\ref{fig:experiment}, at $t = 4$, the Fréchet distance of the human from its yield-left equilibrium trajectory is very high, while for the yield-left equilibrium, the Fréchet distance is low, indicating that the human has chosen the yield-right equilibrium. Using our method, the robot autonomously identifies this mode, moves accordingly, and reaches its goal location while successfully avoiding collision with the human.

We ran 10 different trials, and each time, we asked the human to either yield to their right or left randomly. Our algorithm was able to successfully compute both modes and identify the mode that the human was following. Using our algorithm, the robot was able to successfully avoid collision with the human in real-time and reach its goal location in all 10 runs.

%% file: conclusion.tex
\vspace{-0.3cm}
\section{Conclusion}

In this work, we introduced MultiNash-PF, a novel algorithm for efficiently computing multiple interaction modes. By leveraging potential game theory and implicit particle filtering, MultiNash-PF efficiently identifies coarse estimates of interaction modes, which are then refined using optimization solvers to obtain distinct interaction modes. Our numerical simulations demonstrate that MultiNash-PF significantly reduces computation time by up to 50\% compared to baseline methods, while effectively capturing the multimodal nature of multi-agent interactions. Furthermore, our real-world human-robot interaction experiments highlight the algorithm’s ability to reason about multiple interaction modes and resolve conflicts in real-time. 

% These results underscore the potential of MultiNash-PF to enhance interactive motion planning, enabling more adaptive and socially aware autonomous systems.